\documentclass[11pt, a4paper]{article}

\usepackage[utf8]{inputenc}
\usepackage[T1]{fontenc}
\usepackage{amsmath, amssymb, amsthm, amsfonts}
\newenvironment{fakeproof}
  {\begin{proof}[\textit{``Proof''}]}
  {\end{proof}}
\usepackage{geometry}
\usepackage{hyperref}
\usepackage{graphicx}
\usepackage{mathtools}
\usepackage[capitalize]{cleveref}

\newtheorem{theorem}{Theorem}
\newtheorem{lemma}{Lemma}

\newtheorem{definition}{Definition}

\DeclareMathOperator{\Tr}{Tr}
\DeclareMathOperator{\diag}{diag}
\newcommand{\R}{\mathbb{R}}
\newcommand{\norm}[1]{\left\lVert#1\right\rVert}

\newcommand{\speak}[2]{%
    \vspace{0.5em}\par\noindent\hangindent=2em\hangafter=1%
    \textbf{\textsc{#1}:} #2%
    \par\vspace{0.5em}%
}

\usepackage{listings}
\usepackage{xcolor}

\definecolor{codegreen}{rgb}{0,0.6,0}
\definecolor{codegray}{rgb}{0.5,0.5,0.5}
\definecolor{backcolour}{rgb}{0.95,0.95,0.92}

\newtheorem{prop}{Proposition}

\title{\textbf{Global Convergence of Multiplicative Updates for the Matrix Mechanism: A Collaborative Proof with Gemini 3.}}
\author{Keith Rush \\ \small{Google DeepMind}}
\date{\today}

\begin{document}

\maketitle

\begin{abstract}
We analyze a fixed-point iteration $v \leftarrow \phi(v)$ arising in the optimization of a regularized nuclear norm objective involving the Hadamard product structure, posed in~\cite{denisov} in the context of an optimization problem over the space of algorithms in private machine learning. We prove that the iteration $v^{(k+1)} = \diag((D_{v^{(k)}}^{1/2} M D_{v^{(k)}}^{1/2})^{1/2})$ converges monotonically to the unique global optimizer of the potential function $J(v) = 2 \Tr((D_v^{1/2} M D_v^{1/2})^{1/2}) - \sum v_i$, closing a problem left open there. We also connect this iteration to known results in quantum information and optimal transport.

The bulk of this proof was provided by Gemini 3, subject to some corrections and interventions. Gemini 3 also sketched the initial version of this note. Thus, it represents as much a commentary on the practical use of AI in mathematics as it represents the closure of a small gap in the literature. As such, we include a small narrative description of the prompting process, and some resulting principles (and observed pitfalls) when working with AI to prove mathematics.
\end{abstract}

\section{The Problem and its Brief History}

We consider the problem of finding the fixed point $v^* \in \R_{++}^N$ satisfying the equation:
\begin{equation}
    v^* = \diag\left( \left( D_{v^*}^{1/2} M D_{v^*}^{1/2} \right)^{1/2} \right)
    \label{eq:fixed_point}
\end{equation}
where $M \in \R^{N \times N}$ is a symmetric positive definite matrix ($M \succ 0$), and $D_v = \diag(v)$.

This fixed-point equation arises naturally as the stationarity condition ($\nabla J(v) = 0$) for the potential function $J: \R_{++}^N \to \R$:
\begin{equation}
    J(v) = 2 \Tr\left( \left( D_v^{1/2} M D_v^{1/2} \right)^{1/2} \right) - \sum_{i=1}^N v_i
    \label{eq:potential}
\end{equation}

To our knowledge, the first occurrence of this potential can be found in~\cite{denisov}, where it arose in the context of designing matrix factorizations that are in some sense \emph{optimal} for the matrix mechanism~\cite{mm} under a specific notion of error. \cite{denisov} was able to show that iterating the mapping~\cref{eq:fixed_point} converged locally, and provided empirical evidence that this fixed-point iteration converged extremely rapidly, outpacing pure gradient-based approaches, but global convergence of these iterates was left open.

The followup papers in the series of~\cite{denisov} (see the literature review in~\cite{pillutla2025correlatednoisemechanismsdifferentially}) generally moved away from the unconstrained optimization problem that yields the potential~\cref{eq:potential}, in order to impose structural constraints on matrices which address practical concerns like runtime efficiency and ability to prove privacy amplification results. The specification of the BLT class~\cite{blt} providing the ability to compute near-optimal factorizations with complexity independent of matrix size rendered the convergence of the iterates of~\cref{eq:fixed_point} even more academic. Still, I always viewed it as unsatisfactory that we were unable to prove global convergence despite reasonably serious attempts by reasonably serious mathematicians. This global convergence always had the feel to me of something which was just an issue of putting together the right tools, which probably all existed.

I therefore kept it in my back pocket, and tested it repeatedly on new versions of various AI models as they were released. Given my affiliation, most of these attempts used Gemini. With a few iterations, Gemini 3 was able to identify the crucial piece that we had missed: a variational characterization of the nuclear norm that allows the problem to be diagonalized.

This diagonalization of our crucial object essentially required the ability to call certain linear algebra to hand that neither myself nor my collaborators demonstrated, but it was not exactly revolutionary. This is not unreasonable -- mathematics is quite large, and not every mathematician can recall every inequality or characterization of an object under study, particularly if this object is far from their main field. We hope that this note can serve as a small example of utilization of AI in mathematics, echoing several that have appeared elsewhere in the literature, and also represent the closure of a problem that had always nagged me.

We study the convergence of the parameter-free multiplicative update rule:
\begin{equation}
    v^{(k+1)} = \phi(v^{(k)}) = \diag\left( \left( D_{v^{(k)}}^{1/2} M D_{v^{(k)}}^{1/2} \right)^{1/2} \right)
    \label{eq:update}
\end{equation}

Our main contribution is a proof that this iteration satisfies the monotonic ascent property $J(v^{(k+1)}) \ge J(v^{(k)})$, ensuring global convergence to the unique maximum. We also include discussion of the experience of working interactively with natural-language AI models to produce nontrivial mathematics.

\section{Theoretical Analysis}

\subsection{Basic Properties}

The result here depends critically on a representation of the first term in our potential as the \emph{nuclear norm} of a matrix, and a particular variational characterization for this norm.

\begin{definition}[Nuclear Norm]
For any real matrix $A$, define the nuclear norm $\|A\|_*$ as:
$$
 \|A\|_* = \Tr((A^\top A)^{1/2})
$$
\end{definition}

We will collect the properties of the nuclear norm that we need before we begin.

\begin{lemma}\label{lem:var_nuc}
For any real $A$, the nuclear norm can be characterized in the following manner:
\[
    \| A \|_* = \sup_{U \in \mathcal{U}(N)} \Tr(U^\top A)
\]
where $\mathcal{U}(N)$ is the set of unitary matrices. The supremum is attained when $U$ is the unitary factor from  the polar decomposition of $A$ (see Section 7.4 of~\cite{handj}).

\end{lemma}
\begin{proof}
We recall the polar factorization of $A$, the representation $A = UP$ where $P = (A^\top A)^{1/2}$ and $U$ is unitary. (Since $A$ is real, $U$ is in fact orthogonal). Therefore $U^\top A = P$, and:

$$\|A\|_* = \Tr\left((A^\top A)^{1/2}\right) = \Tr(P) = \Tr(U^\top A).$$

This shows that $\|A\|_* \leq \sup_{U \in \mathcal{U}(N)} \Tr(U^\top A)$. The reverse inequality follows from the Von Neumann trace inequality; for $O$ unitary, its singular values are of absolute value 1, and therefore
$$|\Tr(O^\top A)| \leq \sum |\omega_i \alpha_i| = \sum |\alpha_i| = \Tr\left((A^\top A)^{1/2}\right)$$

\end{proof}

\begin{lemma}[Strict Concavity]\label{lem:concave}
For a symmetric positive definite matrix $M \succ 0$, the potential function 
\[
J(v) = 2 \Tr\left( \left( D_v^{1/2} M D_v^{1/2} \right)^{1/2} \right) - \sum_{i=1}^N v_i
\]
is strictly concave on the positive orthant $\R_{++}^N$.
\end{lemma}

\begin{proof}
For any matrices $A$ and $B$, $AB$ and $BA$ share the same non-zero eigenvalues (this is~\cite{handj}, Theorem 1.3.22). Let $A = D_v^{1/2}M^{1/2}$ and $B = M^{1/2} D_v^{1/2}$.

As $AB$ and $BA$ share eigenvalues, the trace of their square roots is identical:

$$\Tr\left((D_v^{1/2} M D_v^{1/2})^{1/2}\right) = \Tr\left((M^{1/2} D_v M^{1/2})^{1/2}\right)$$

Since the mapping $v \mapsto M^{1/2} D_v M^{1/2}$ is linear in $v$ and injective, Lieb's concavity theorem (\cite{LIEB1973267}) combined with concavity for a composition of an injective linear map with a concave map finish the proof.
\end{proof}

\begin{lemma}[Coercivity]
The potential function $J(v)$ is coercive on $\R_{+}^N$ in the sense that $J(v) \to -\infty$ as $\norm{v}_1 \to \infty$.
\end{lemma}

\begin{proof}
Let $A = D_v^{1/2} M D_v^{1/2}$. Note that $\Tr(A^{1/2})$ is the nuclear norm of $D_v^{1/2} M^{1/2}$. By the relationship between the nuclear norm and the Frobenius norm, $\Tr(A^{1/2}) \le \sqrt{N \Tr(A)}$. We have:
\begin{equation*}
    \Tr(A) = \Tr(D_v^{1/2} M D_v^{1/2}) = \Tr(M D_v) = \sum_{i=1}^N M_{ii} v_i \le \max_i(M_{ii}) \norm{v}_1
\end{equation*}
Thus, $J(v) \le 2\sqrt{N \max_i(M_{ii}) \norm{v}_1} - \norm{v}_1$. As $\norm{v}_1 \to \infty$, the linear term dominates the square-root term, hence $J(v) \to -\infty$. This, combined with strict concavity and the fact that $J(v)$ is bounded at the boundary $\partial \R_{+}^N$, ensures the existence of a unique maximizer $v^* \in \R_{++}^N$.
\end{proof}

\begin{lemma}[Positive Invariance]
If $v^{(k)} \in \R_{++}^N$ and $M \succ 0$, then $v^{(k+1)} = \phi(v^{(k)}) \in \R_{++}^N$.
\end{lemma}

\begin{proof}
Let $X = D_{v^{(k)}}^{1/2} M D_{v^{(k)}}^{1/2}$. Since $v^{(k)} > 0$ element-wise, $D_{v^{(k)}}^{1/2}$ is a non-singular diagonal matrix. Because $M \succ 0$, the congruence transformation $X$ is also symmetric positive definite ($X \succ 0$). The unique positive definite square root $X^{1/2}$ is also strictly positive definite. Since the diagonal entries of any strictly positive definite matrix are strictly positive, $v^{(k+1)}_i = [X^{1/2}]_{ii} > 0$ for all $i$.
\end{proof}

Counterintuitive as it may seem given its triviality, we will begin by analyzing the scalar case. Surprisingly, using the nuclear norm representation above, we will be able to, in effect, reduce to it.

\subsection{The Scalar Case ($N=1$)}

To build intuition, we first analyze the scalar case where $M = 1$ and $D_v = x$. The potential simplifies to $J(x) = 2\sqrt{x} - x$. The update rule is $x_{new} = \sqrt{x}$.

Let $\Delta J = J(x_{new}) - J(x)$. Substituting $x_{new} = x^{1/2}$:
\begin{align*}
    \Delta J &= (2 x^{1/4} - x^{1/2}) - (2 x^{1/2} - x) \\
             &= 2 x^{1/4} - 3 x^{1/2} + x
\end{align*}
To analyze the sign of $\Delta J$, we perform a change of variables $x = z^4$ (with $z > 0$). This transforms the expression into a polynomial in $z$:
\begin{equation}\label{eq:scalar_ascent}
    \Delta J(z) = z^4 - 3z^2 + 2z
\end{equation}
By inspecting the root at $z=1$ (the fixed point), we can factorize the polynomial:
\begin{equation}
    \Delta J(z) = z(z-1)^2(z+2)
\end{equation}
Since $z > 0$, we have $z > 0$, $(z-1)^2 \ge 0$, and $(z+2) > 0$. Thus, $\Delta J \ge 0$, with equality if and only if $z=1$ ($x=1$). This factorization confirms unconditional monotonicity for the scalar case.

\subsection{The Matrix Case}

\begin{theorem}
For positive-definite $M$, the algorithm defined by iterating the map $\phi$,

\begin{equation}
    \phi(v) = \diag\left( \left( D_v^{1/2} M D_v^{1/2} \right)^{1/2} \right),
\end{equation}

converges to the unique fixed point of $\phi$ for any initialization in $\mathbb{R}^N_{++}$.
\end{theorem}

\begin{proof}

In the general matrix case, $D_v$ and $M$ do not commute. Consequently, we cannot apply the scalar polynomial factorization directly to the eigenvalues. Note that the fixed point of $\phi$ is exactly the maximizer of the potential~\cref{eq:potential}. To prove convergence of the iterates of $\phi$, we employ the \textbf{Majorization-Minimization (MM)} principle by constructing a variational surrogate function: we will show that this surrogate function lower bounds the potential $J$, and iterating $\phi$ makes strict progress on this surrogate function every step, and thus that the iterates make monotonic progress on the potential. Elementary analysis allows us to conclude that the iterates indeed converge to $v^*$: given coercivity, since we stay away from the boundary (IE, no $v_i \to 0$ under this iteration, as can be shown simply), we can extract a subsequence which converges to a limit point, but this could be nothing but a fixed point. And monotonic ascent tells us that we cannot leave a neighborhood once we have entered it. Alternatively, we may invoke Zangwill's theorem~\cite{zangwill1969nonlinear} to conclude.

Thus we proceed with the MM framework.

\subsubsection{Step 1: Construction of the Minorizer (Surrogate Function)}\label{sec:minorizer_construction}

Let $A(v) =  M^{1/2} D_v^{1/2}$. By~\cref{lem:var_nuc}, the first term of our potential is:
\begin{equation}
    2 \Tr\left( (D_v^{1/2} M D_v^{1/2})^{1/2} \right) = 2 \| M^{1/2} D_v^{1/2} \|_* = \sup_{U \in \mathcal{U}(N)} 2 \Tr\left( U^\top M^{1/2} D_v^{1/2} \right)
\end{equation}

Let $v^{(k)}$ be the current iterate. Define the matrix $A_k = M^{1/2} D_{v^{(k)}}^{1/2} $. We compute the polar decomposition $A_k = U_k P_k $, where $P_k = (A_k^\top  A_k)^{1/2}$ is positive semidefinite and $U_k$ is unitary. Note that strictly $U_k$ aligns the singular vectors; implicitly this solves the variational problem $\sup_U \Tr(U^\top A_k)$.

We define the surrogate function $G(v; v^{(k)})$ by fixing the unitary $U_k$:
\begin{equation}
    G(v; v^{(k)}) = 2 \Tr\left( U_k^\top M^{1/2} D_v^{1/2}  \right) - \sum_{i=1}^N v_i
\end{equation}

\textbf{Properties of the Surrogate:}
\begin{enumerate}
    \item \textbf{Lower Bound:} By the variational definition of the nuclear norm, $J(v) \ge G(v; v^{(k)})$ for all $v$.
    
    \item \textbf{Tightness (Value Match):} At $v = v^{(k)}$, the optimal unitary for the variational form is precisely $U_k$. Thus, $J(v^{(k)}) = G(v^{(k)}; v^{(k)})$.
    
    \item \textbf{Tangency (Gradient Match):} We verify that the surrogate shares the same first-order behavior as the objective at $v^{(k)}$.
    
    The gradient of the original potential $J(v)$ is:
    \begin{equation}
        \nabla J(v) = \diag\left( (D_v^{1/2} M D_v^{1/2})^{1/2} \right) \oslash v - \mathbf{1}
    \end{equation}
    At $v^{(k)}$, note that $(D_{v^{(k)}}^{1/2} M D_{v^{(k)}}^{1/2})^{1/2} = (A_k^\top A_k)^{1/2} = P_k$. Using the polar decomposition $P_k =U_k^\top A_k$:
    \begin{equation}
        \nabla J(v^{(k)}) = \diag(U_k^\top A_k) \oslash v^{(k)} - \mathbf{1}
    \end{equation}
    
    Now, consider the gradient of the surrogate $G(v; v^{(k)})$. The trace term is linear in $D_v^{1/2}$:
    \begin{equation}
        2 \Tr(U_k^\top M^{1/2} D_v^{1/2}) = 2 \sum_{i} [U_k^\top M^{1/2}]_{ii} \sqrt{v_i}
    \end{equation}
    The gradient with respect to $v_i$ is:
    \begin{equation}
        \frac{\partial G}{\partial v_i} = \frac{[ U_k^\top M^{1/2} ]_{ii}}{\sqrt{v_i}} - 1
    \end{equation}
    Evaluating at $v^{(k)}$, we substitute $[U_k^\top M^{1/2} ]_{ii}$ back into matrix form. Note that, as $A_k = M^{1/2} D_{v^{(k)}}^{1/2}$, we have $\diag( U_k^\top A_k )_i =  [U_k^\top M^{1/2}]_{ii} \sqrt{v_i^{(k)}}$. Thus:
    \begin{equation}
        \nabla G(v^{(k)}; v^{(k)}) = \frac{\diag(U_k^\top A_k )}{\sqrt{v^{(k)}} \odot \sqrt{v^{(k)}}} - \mathbf{1} = \nabla J(v^{(k)})
    \end{equation}
    This confirms $v^{(k)}$ is a valid touch-point for Majorization-Minimization.
\end{enumerate}

\subsubsection{Step 2: Monotonic Ascent via the Surrogate}

To prove monotonic ascent, we analyze the behavior of the surrogate function constructed in the previous step. Recall that the surrogate, constructed by freezing the unitary $U_k$ from the polar decomposition of $M^{1/2} D_{v^{(k)}}^{1/2}$, is given by:
\begin{equation}\label{eq:g_diag}
    G(v; v^{(k)}) = \sum_{i=1}^N \left( 2 c_i \sqrt{v_i} - v_i \right)
\end{equation}
where the coefficients are $c_i = \left[U_k^\top  M^{1/2}  \right]_{ii}$.

\paragraph{The Surrogate Maximizer vs. The Algorithm Update.}
The global maximizer of this separable surrogate, denoted $v^{opt}$, is found by setting the derivative to zero:
\begin{equation}
    \frac{\partial G}{\partial v_i} = \frac{c_i}{\sqrt{v_i}} - 1 = 0 \implies v_i^{opt} = c_i^2
\end{equation}
However, we must verify how this relates to our actual fixed-point update $\phi(v^{(k)})$.
The gradient of the true potential $J(v)$ at $v^{(k)}$ is given by $\nabla J(v) = \phi(v) \oslash v - \mathbf{1}$.
The gradient of the surrogate $G(v)$ at $v^{(k)}$ is $\nabla G(v) = c \oslash \sqrt{v} - \mathbf{1}$.
By the tangency condition above, $\nabla J(v^{(k)}) = \nabla G(v^{(k)}; v^{(k)})$. Equating the terms reveals the explicit form of the update in terms of the surrogate coefficients:
\begin{equation}\label{eq:v_opt_ref}
    \frac{\phi(v^{(k)})_i}{v^{(k)}_i} = \frac{c_i}{\sqrt{v^{(k)}_i}} \implies v^{(k+1)}_i = c_i \sqrt{v^{(k)}_i}
\end{equation}
Substituting $c_i = \sqrt{v_i^{opt}}$, we observe that the algorithm update is the \textbf{element-wise geometric mean} of the current iterate and the surrogate's optimum:
\begin{equation}
    v^{(k+1)}_i = \sqrt{v^{(k)}_i \cdot v_i^{opt}}
\end{equation}

\paragraph{Proof of Ascent.}
Since the algorithm does not jump strictly to the maximizer $v^{opt}$ but instead takes a damped step, we must show that this step still increases the surrogate value.
Let $\Delta G = G(v^{(k+1)}) - G(v^{(k)})$. Since the terms separate, we analyze the scalar contribution for a single index $i$ (omitting the subscript for brevity). Let $v$ be the current value and $c$ be the coefficient.
The update is $v_{new} = c\sqrt{v}$.
\begin{align*}
    \Delta G_{scalar} &= \left( 2 c \sqrt{v_{new}} - v_{new} \right) - \left( 2 c \sqrt{v} - v \right) \\
    &= \left( 2 c \sqrt{c\sqrt{v}} - c\sqrt{v} \right) - 2 c \sqrt{v} + v \\
    &= 2 c^{3/2} v^{1/4} - 3 c v^{1/2} + v
\end{align*}
And we see the relation to the scalar case above. Again we substitute variables. Let $v = c^2 z^4$ (for $z > 0$). This implies $v^{1/2} = c z^2$ and $v^{1/4} = c^{1/2} z$.
\begin{align*}
    \Delta G_{scalar} &= 2 c^{3/2} (c^{1/2} z) - 3 c (c z^2) + c^2 z^4 \\
    &= c^2 \left( 2z - 3z^2 + z^4 \right)
\end{align*}
We must factorize the polynomial $P(z) = z^4 - 3z^2 + 2z$. But note, we already did this, in the scalar case.
Thus, referring to~\cref{eq:scalar_ascent}, we have
\begin{equation}
    \Delta G_{scalar} = c^2 \cdot z \cdot (z-1)^2 \cdot (z+2)
\end{equation}
Since $c, z > 0$ (where $c > 0$ relies on positive-definiteness of $M$) and $(z-1)^2 \ge 0$, we have $\Delta G \ge 0$ unconditionally.

Combining these results yields the ascent inequality chain:
\begin{equation}
    J(v^{(k+1)}) \ge G(v^{(k+1)}; v^{(k)}) \ge G(v^{(k)}; v^{(k)}) = J(v^{(k)})
\end{equation}
This confirms that the update rule guarantees monotonic ascent of the objective function, with strict increase whenever $v^{(k)}$ is not a fixed point (as, by~\cref{eq:v_opt_ref}, $z=1$ again corresponds to the fixed point).
\end{proof}

\textbf{Remark:} This same proof can show linear (exponential) convergence to the objective. But this was already shown asymptotically (IE, close to the fixed point) by~\cite{denisov}, and convergence rates are typically formulated asymptotically. 

\subsection{Jumping to the optimum: connections to optimal transport.}\label{sec:opt_transport}

One may ask, what happens if we instead jump to the \emph{exact} optimum of the surrogate function? Clearly such an algorithm must similarly converge globally, by an identical argument. And since we computed the derivative of $G$ exactly above, we may immediately specify this version.

\begin{prop}\label{prop:max_minorizer}

The algorithm defined by:

$$
v_{\text{new}} = \phi(v)^2 \oslash v
$$

is similarly globally convergent to the unique fixed point of $\phi$.
\end{prop}
\begin{proof}
    Here it is enough to show that this specification zeroes out the gradient of $G$ as a function of its first argument. But this is immediate; zeroing out the gradient of~\cref{eq:g_diag} in $v$ is equivalent to simultaneously solving the equations
    $$\frac{c_i}{\sqrt{v_i}} - 1 = 0$$
    for all $i$, where $c_i$ is defined in step 2 above. That is, letting $v_i = c_i^2$. Because $\phi(v)_i = c_i \sqrt{v_i}$, we see:

    $$
    v_{\text{new}} = \textbf{c} ^2 = \frac{(\textbf{c} \odot \sqrt{v})^2}{v} = \frac{\phi(v)^2}{v}
    $$
\end{proof}

~\cref{prop:max_minorizer}, however, is in fact derivable from the literature in optimal transport and quantum information theory.

Observe that our potential function $J(v)$ in~\cref{eq:potential} is deeply related to the squared \emph{Bures-Wasserstein distance} (which corresponds to the 2-Wasserstein distance between centered Gaussian distributions, and the Bures metric between quantum states). For two symmetric positive definite matrices $A$ and $B$, this distance is defined as~\cite{bhatia2019bures}:
\begin{equation}
    d_{BW}^2(A, B) = \Tr(A) + \Tr(B) - 2 \Tr\left( (A^{1/2} B A^{1/2})^{1/2} \right)
\end{equation}
Setting $A = D_v$ and $B = M$, we find:
\begin{equation}
    d_{BW}^2(D_v, M) = \sum_{i=1}^N v_i + \Tr(M) - 2 \Tr\left( \left( D_v^{1/2} M D_v^{1/2} \right)^{1/2} \right) = \Tr(M) - J(v)
\end{equation}
Because $\Tr(M)$ is a constant, maximizing $J(v)$ over the positive orthant is mathematically equivalent to computing the Bures-Wasserstein projection of the matrix $M$ onto the manifold of diagonal matrices. In quantum information theory, this is exactly the problem of finding the closest classical (or incoherent) state to a given density matrix $M$, a foundational calculation for quantifying quantum coherence~\cite{baumgratz2014quantifying}.

In the optimal transport literature, minimizing the Bures-Wasserstein distance is typically achieved via Wasserstein Gradient Descent (WGD)~\cite{chewi2020gradient}. To project a Gaussian with covariance $D_v$ onto a target $M$, the optimal transport map is $T = D_v^{-1/2} (D_v^{1/2} M D_v^{1/2})^{1/2} D_v^{-1/2}$. To constrain the descent to the diagonal manifold, we project this transport map onto the diagonal, yielding $T_{\text{diag}} = \diag(T) = D_v^{-1/2} D_{\phi(v)} D_v^{-1/2}$. The WGD step with a step size of $\eta=1$ applies this map to the current covariance via congruence:
\begin{equation}
    D_{v_{\text{new}}} = T_{\text{diag}} D_v T_{\text{diag}} = D_v^{-1/2} D_{\phi(v)}^2 D_v^{-1/2}
\end{equation}
Extracting the diagonal of this matrix yields exactly $v_{\text{new}} = \phi(v)^2 \oslash v$. Because WGD with a unit step size is well-studied and geometrically supported on the Bures-Wasserstein manifold, its rapid global convergence is firmly guaranteed by existing literature~\cite{chewi2020gradient}.

\paragraph{Does existing literature imply the convergence of the original algorithm?}

The short answer is no (as far as we are aware), but it also highlights that the original formulation in~\cite{denisov} was likely suboptimal.

As we noted above, there is a simple relationship between the algorithm of~\cite{denisov} and the WGD step:
\begin{equation}
    v_{\text{orig}} = \phi(v) = \sqrt{ v \odot \frac{\phi(v)^2}{v} } = \sqrt{ v \odot v_{\text{new}} }
\end{equation}
The original iteration $v \leftarrow \phi(v)$ is precisely the \emph{element-wise geometric mean} of the current iterate and the optimal jump. In the Log-Euclidean geometry of positive definite matrices, this corresponds to an under-relaxed step with a size of $\eta = 1/2$. 

While the optimal transport literature provides deep theoretical guarantees that the full WGD step ($\eta=1$) monotonically decreases the Bures-Wasserstein distance, taking a ``damped'' half-step along a Log-Euclidean geodesic does not trivially guarantee strict monotonic ascent of the objective. Standard geodesic convexity arguments fail because the damping occurs in an entirely different geometry than the objective function.

\paragraph{An empirical comment.}

In principle, there is not a clear relationship between the empirics of the algorithm presented in~\cite{denisov} and that of~\cref{prop:max_minorizer}. Empirically, the jump iteration $v \leftarrow \phi(v)^2 \oslash v$ of~\cref{prop:max_minorizer} requires nearly exactly half as many iterations to reach high precision asymptotically compared to the original iteration $v \leftarrow \phi(v)$, though they are equivalent in other respects. 

This factor of two can be derived by examining the Jacobians of the two updates at the fixed point $v^*$. Let $J_\phi$ be the Jacobian of $\phi$ at $v^*$. By differentiating the jump update $\psi(v) = \phi(v)^2 \oslash v$ and evaluating at $v^*$ (where $\phi(v^*) = v^*$), one easily finds its Jacobian to be $J_\psi = 2J_\phi - I$. 

Consequently, if the slowest-decaying error mode of the original algorithm has an eigenvalue $\lambda = 1 - c$, the jump algorithm will contract that same mode with eigenvalue $2(1-c) - 1 = 1 - 2c$. For small $c$, the linear rate constant $-\log(1-2c) \approx 2c$ is exactly double that of $-\log(1-c) \approx c$. 

This relationship is most elegantly seen in the log-domain. Letting $y = \log v$ (element-wise), the jump algorithm takes the form $y_{\text{new}} = y + 2(\log \phi(e^y) - y)$. If we interpret $\log \phi(e^y) - y$ as an implicit gradient direction, jumping to the surrogate's optimum corresponds to taking a full step ($\eta = 1$), whereas the original algorithm is precisely an \emph{under-relaxed} step ($\eta = 1/2$).

In contrast to these two algorithms, which take steps fundamentally influenced by the underlying geometry, directly parameterizing the vector $v$ and taking gradient steps in Euclidean space achieves nowhere near the same asymptotic rate. This echoes the observations of~\cite{denisov}, which observed that even augmented with Newton steps, Euclidean gradient descent could not match the convergence of the original iterations of $\phi$.

\begin{figure}[htbp]
    \centering
    \includegraphics[width=0.8\textwidth]{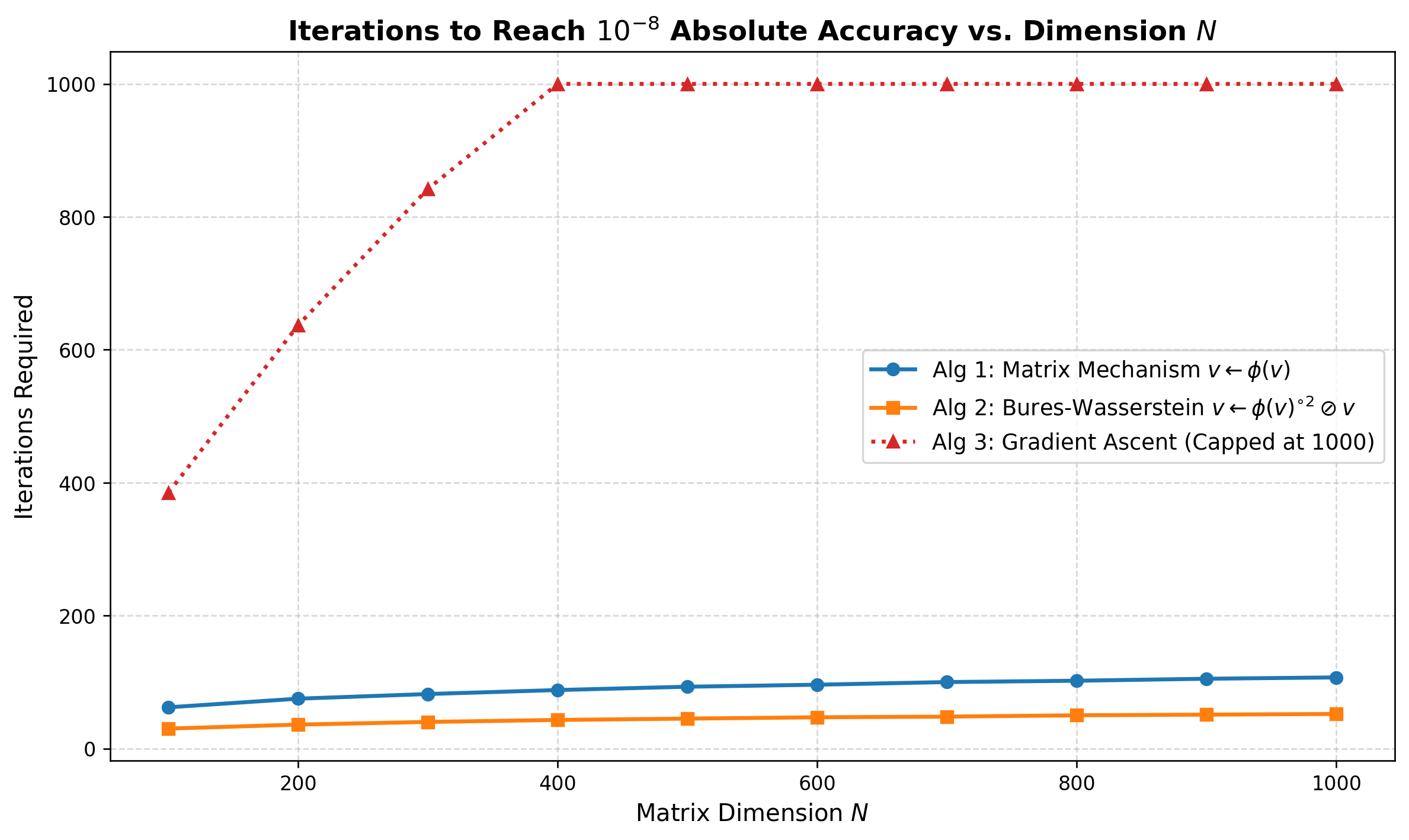}
    \caption{Comparison of steps to reach $1e-8$ accuracy between the original iteration of~\cite{denisov}, that of~\cref{prop:max_minorizer} (here called ``Bures-Wasserstein''), and gradient ascent equipped with Armijo step-size selection and backtracking.}
    \label{fig:iterations_phi_ot}
\end{figure}

~\cref{fig:iterations_phi_ot} shows the empirical relationships between the Denisov iteration, that of~\cref{prop:max_minorizer}, and Euclidean gradient ascent augmented with Armijo step-size selection and backtracking. 

\section{Conclusion}

I don't consider the result shown here to be particularly noteworthy; I view it as the closure of a small problem that appeared in one of my past papers. It was something of a shame to me that we did not prove global convergence initially, despite spending a nontrivial amount of time searching for a surrogate function to fit in the MM framework (and some reasonably heavy analysis by Sergey Denisov which showed $\phi$ was a local contraction around its fixed point). The key fact we were missing was the representation~\cref{lem:var_nuc}, which, as we saw above, allowed us to effectively diagonalize the problem. The effectiveness of a potential method of ``diagonalizing the problem'' was quite clear to us; but like most ex-mathematicians working in industry, I would be hard-pressed to recall every matrix equation I have ever seen.

I have a deep and abiding love for mathematics; I entered computing because I believed that it would be a significant source of problems and approaches for the mathematics of the 21st century, in roughly the way physics was for so many years. I remember seeing Jordan Ellenberg at the 2025 joint mathematics meetings, and discussing with him my fear that advances in AI would be used primarily for companionship chatbots rather than progress in fundamental mathematics and science. He rightly pointed out to me the slightly strange priorities my comment implied: that I would rather prove more theorems than bring some comfort to lonely people around the world!

There is one particularly interesting feature of the present note from my perspective. I have never worked in optimal transport or quantum information; when I initially asked on MathOverflow if this iteration had been studied~\cite{rush_mo}, there was an observed connection to so-called ``quantum fidelity'', but making this connection any more precise was not obvious to me. Thus the content of~\cref{sec:opt_transport} is effectively entirely due to the improved pattern-matching capabilities of contemporary AI systems, where one may query in more or less an inverse manner compared to what was previously possible: instead of knowing what you're searching for, we can probe by asking what we \emph{should} be looking for. Indeed, from this perspective, it is perhaps unsurprising that in 2021 and 2022 we stumbled upon an algorithm that was fundamentally related to something from optimal transport, but not identical -- and we had, at the time, no way to easily identify this relationship. 

I initially did this work in December of 2025; since then, several other examples of work of this flavor have appeared (\cite{knuth2026claudecycles,openai2026firstproof,guevara2026gluon,bryan2026motivic}). I am happy to see that greater minds than my own have found utility in the current models in a roughly similar way to me.

I decided to write up this note, rather than simply let sit the fact that I was able to close a not-entirely-trivial problem primarily through prompting of an AI model, in the hope that I might be able to play some small part in bringing the communities of pure mathematics and AI closer in these, the early days of our journey together. 

\section{Acknowledgments}

Gemini 3 made me aware of the representation~\cref{lem:var_nuc}, and sketched the major steps of this proof. I filled in some details and fixed some argumentation (including some subtle gaps), but I consider this proof to have been effectively fully generated by Gemini 3. Moreover, Gemini 3 wrote a significant amount of this note; after some rounds of edits, it is difficult to say exactly how much, but its contribution was substantial.

I owe a debt to Abhradeep Thakurta and Brendan McMahan, with whom I worked on formulating the problem that eventually led here. Certainly also our other co-authors on that first paper, Adam Smith and my advisor, Sergey Denisov.

Finally, thank you to my wife Breanne Lynch, who has provided me with unfailing support, even if I can't quite articulate \emph{why} I need to be typing away at odd hours.

\bibliographystyle{alpha}
\bibliography{references}

@inproceedings{denisov,
 author = {Denisov, Sergey and McMahan, H. Brendan and Rush, John and Smith, Adam and Guha Thakurta, Abhradeep},
 booktitle = {Advances in Neural Information Processing Systems},
 editor = {S. Koyejo and S. Mohamed and A. Agarwal and D. Belgrave and K. Cho and A. Oh},
 pages = {5910--5924},
 publisher = {Curran Associates, Inc.},
 title = {Improved Differential Privacy for SGD via Optimal Private Linear Operators on Adaptive Streams},
 url = {https://proceedings.neurips.cc/paper_files/paper/2022/file/271ec4d1a9ff5e6b81a6e21d38b1ba96-Paper-Conference.pdf},
 volume = {35},
 year = {2022}
}

@article{mm,
  author       = {Chao Li and
                  Gerome Miklau and
                  Michael Hay and
                  Andrew McGregor and
                  Vibhor Rastogi},
  title        = {The matrix mechanism: optimizing linear counting queries under differential
                  privacy},
  journal      = {{VLDB} J.},
  volume       = {24},
  number       = {6},
  pages        = {757--781},
  year         = {2015},
  url          = {https://doi.org/10.1007/s00778-015-0398-x},
  doi          = {10.1007/S00778-015-0398-X},
  bibsource    = {dblp computer science bibliography, https://dblp.org}
}

@misc{pillutla2025correlatednoisemechanismsdifferentially,
      title={Correlated Noise Mechanisms for Differentially Private Learning}, 
      author={Krishna Pillutla and Jalaj Upadhyay and Christopher A. Choquette-Choo and Krishnamurthy Dvijotham and Arun Ganesh and Monika Henzinger and Jonathan Katz and Ryan McKenna and H. Brendan McMahan and Keith Rush and Thomas Steinke and Abhradeep Thakurta},
      year={2025},
      eprint={2506.08201},
      archivePrefix={arXiv},
      primaryClass={cs.LG},
      url={https://arxiv.org/abs/2506.08201}, 
}

@INPROCEEDINGS {blt,
author = { Dvijotham, Krishnamurthy Dj and McMahan, H. Brendan and Pillutla, Krishna and Steinke, Thomas and Thakurta, Abhradeep },
booktitle = { 2024 IEEE 65th Annual Symposium on Foundations of Computer Science (FOCS) },
title = {{ Efficient and Near-Optimal Noise Generation for Streaming Differential Privacy }},
year = {2024},
volume = {},
ISSN = {},
pages = {2306-2317},
doi = {10.1109/FOCS61266.2024.00135},
url = {https://doi.ieeecomputersociety.org/10.1109/FOCS61266.2024.00135},
publisher = {IEEE Computer Society},
address = {Los Alamitos, CA, USA},
month =Oct}

@book{handj,
  asin = {0521386322},
  author = {Horn, Roger A. and Johnson, Charles R.},
  description = {Amazon.com: Matrix Analysis (9780521386326): Roger A. Horn, Charles R. Johnson: Books},
  dewey = {512},
  ean = {9780521386326},
  isbn = {0521386322},
  publisher = {Cambridge University Press},
  title = {Matrix Analysis},
  url = {http://www.amazon.com/Matrix-Analysis-Roger-Horn/dp/0521386322%3FSubscriptionId%3D192BW6DQ43CK9FN0ZGG2%26tag%3Dws%26linkCode%3Dxm2%26camp%3D2025%26creative%3D165953%26creativeASIN%3D0521386322},
  year = 1990
}

@book{zangwill1969nonlinear,
  title={Nonlinear Programming: A Unified Approach},
  author={Zangwill, W.I.},
  isbn={9780136235798},
  lccn={69010606},
  series={Prentice-Hall international series in management},
  url={https://books.google.com/books?id=TWhxLcApH9sC},
  year={1969},
  publisher={Prentice-Hall}
}

@article{GuerraPujol2012,
  author  = {Guerra-Pujol, F. E.},
  title   = {Gödel's Loophole},
  journal = {SSRN Electronic Journal},
  year    = {2012},
  doi     = {10.2139/ssrn.2010183},
  url     = {https://papers.ssrn.com/sol3/papers.cfm?abstract_id=2010183}
}

@article{LIEB1973267,
title = {Convex trace functions and the Wigner-Yanase-Dyson conjecture},
journal = {Advances in Mathematics},
volume = {11},
number = {3},
pages = {267-288},
year = {1973},
issn = {0001-8708},
doi = {https://doi.org/10.1016/0001-8708(73)90011-X},
url = {https://www.sciencedirect.com/science/article/pii/000187087390011X},
author = {Elliott H Lieb}
}

@misc{knuth2026claudecycles,
  author       = {Donald E. Knuth},
  title        = {Claude's Cycles},
  howpublished = {Unpublication available at \url{https://www-cs-faculty.stanford.edu/~knuth/papers/claude-cycles.pdf}},
  month        = feb,
  year         = {2026},
  note         = {Revised March 16, 2026},
  institution  = {Stanford University, Computer Science Department}
}

@misc{openai2026firstproof,
  author       = {{OpenAI}},
  title        = {Our First Proof submissions},
  howpublished = {OpenAI Blog},
  month        = feb,
  year         = {2026},
  url          = {https://openai.com/index/first-proof-submissions/},
  note         = {Published February 20, 2026}
}

@article{guevara2026gluon,
  author       = {Alfredo Guevara and Alexandru Lupsasca and David Skinner and Andrew Strominger and Kevin Weil},
  title        = {Single-minus gluon tree amplitudes are nonzero},
  journal      = {arXiv preprint arXiv:2602.12176},
  year         = {2026},
  month        = feb,
  eprint       = {2602.12176},
  archivePrefix = {arXiv},
  primaryClass = {hep-th},
  url          = {https://arxiv.org/abs/2602.12176},
  note         = {On behalf of OpenAI}
}

@article{bryan2026motivic,
  author       = {Jim Bryan and Balázs Elek and Freddie Manners and George Salafatinos and Ravi Vakil},
  title        = {The motivic class of the space of genus 0 maps to the flag variety},
  journal      = {arXiv preprint arXiv:2601.07222},
  year         = {2026},
  month        = jan,
  eprint       = {2601.07222},
  archivePrefix = {arXiv},
  primaryClass = {math.AG},
  url          = {https://arxiv.org/abs/2601.07222},
  note         = {In collaboration with the Google DeepMind Blueshift team, including Adam Brown and Vinay Ramasesh}
}

@article{bhatia2019bures,
  title={On the {Bures-Wasserstein} distance between positive definite matrices},
  author={Bhatia, Rajendra and Jain, Tanvi and Lim, Yongdo},
  journal={Expositiones Mathematicae},
  volume={37},
  number={2},
  pages={165--191},
  year={2019},
  publisher={Elsevier}
}

@article{baumgratz2014quantifying,
  title={Quantifying coherence},
  author={Baumgratz, Tillmann and Cramer, Marcus and Plenio, Martin B},
  journal={Physical Review Letters},
  volume={113},
  number={14},
  pages={140401},
  year={2014},
  publisher={APS}
}

@inproceedings{chewi2020gradient,
  title={Gradient descent algorithms for {Bures-Wasserstein} barycenters},
  author={Chewi, Sinho and Maunu, Tyler and Rigollet, Philippe and Stromme, Austin J},
  booktitle={Conference on Learning Theory},
  pages={1276--1304},
  year={2020},
  organization={PMLR}
}

@misc{rush_mo,
  author       = {Rush, Keith},
  title        = {Conjecture on convergence of iterated near-matrix square root},
  howpublished = {MathOverflow},
  year         = {2022},
  month        = {jan},
  note         = {URL: \url{https://mathoverflow.net/q/413633} (version: 2023-01-25)},
  url          = {https://mathoverflow.net/q/413633}
}

\appendix

\section{A User's Journey}

Given my background, I am relatively uniquely positioned to effectively interact with current AI systems and produce effective mathematics. At present, it is a nontrivial procedure, and it is not entirely surprising that, with a different mental model of contemporary AI systems, interactions may go nowhere useful. In this section, I will attempt to extract the relevant aspects of my experience resulting in the proof above, and highlight some of the ways the process could have gone wrong.

First, it must be stated that the models that currently exist are generally \emph{incredibly credulous}. Mathematics is a process which requires surfacing all relevant reasoning steps. As much as formal methods indicate that we may be somewhat deceiving ourselves about the tightness of our reasoning, mathematics is, in my experience, the domain with the highest bar of reasoning in human activity.

Thus, keeping in mind the view of LLMs as a ``steerable probability distribution over all human text'' (which is probably correct to first order), it is no surprise that models will tend, from the mathematician's perspective, to leap to conclusions. These leaps, like our leaps, are often incorrect. Currently, models don't demonstrate significant ability to spot their own gaps and critique their own leaps, though I expect this to change in the near future. This leads to the first pattern of effective use that I have found:

\textbf{Principle 1: Push hard on gaps.} Models will often leap over the requirements of strict reasoning. Mathematicians know this is common practice in human interaction, and we must work quite hard to achieve the state of mind that allows for precise reasoning. Mathematicians, then, are often excellent (though far from perfect!) at ``sniffing out'' logical gaps. (In an infamous story, the great logician Kurt G\"{o}del was said to have found an ``internal contradiction'' in the US Constitution which permits the US to legally convert itself from a republic to a dictatorship~\cite{GuerraPujol2012}. Luckily the judge administering the citizenship examination to G\"{o}del declined to pursue the matter.) 

Interacting with an AI with the goal of proving mathematics often stretches this logical-gap-sniffing capacity to the limit -- there are so many! I will give one example that I found in the log of my interaction with Gemini that resulted in this note. After I pointed to~\cite{denisov} and asked for a proof of the convergence of iterates of $\phi$, Gemini responded with:

\speak{Gemini}{...
Scalar Analogy:
Consider the 1D case $J(x) = 2\sqrt{ax} - x$. The iteration is $x_{new} = \sqrt{ax}$.
Let $u = (x/a)^{1/4}$. The gain $\Delta J = J(x_{new}) - J(x)$ factors as:
$$ \Delta J \propto u(u-1)^2(u+2) $$
This is strictly positive for all $u \neq 1$. The matrix case satisfies a similar inequality.}

Clearly, this is insufficient. (Though, as we have seen above, indeed correct). Thus I asked:

\speak{Keith}{
It seems like the crux of your argument is step 3, particularly, "the matrix case satisfies a similar inequality". please cut no corners and carefully show step 3. then: is step 3 a well-known fact? if so, please give a reference.}

Unfortunately, even to this I was left with a ``matrix case satisfies a similar inequality'' response. Gemini even admitted it could not prove this to me, and suggested I try numerics. At this point I opened a new chat, pasting in the argument Gemini gave me as context, asking:

\speak{Keith}{It seems like the crux of the argument to follow is step 3, particularly, "the matrix case satisfies a similar inequality". please cut no corners and carefully show step 3. then: is step 3 a well-known fact? if so, please give a reference.}

Here we finally, in hindsight, started to get closer. But much reasoning remained suppressed:

\speak{Gemini}{... Summary:
The scalar proof ($z(z-1)^2(z+2) \ge 0$) provides the intuition. The matrix case holds because the Trace acts on the singular values, and the non-commutativity terms are handled by the unitary invariance of the Nuclear Norm, effectively "diagonalizing" the problem into the scalar case at the optimal alignment.}

By now we see implicit references to the proof above, but it was quite difficult in the moment to disambiguate a correct hand-waving around unitary invariance from a pure hallucination.

\speak{Keith}{so, you haven't given me  a proof, right? Please try harder.}

This sequence of interactions yielded the representation of the nuclear norm we have in this note, and the basic outline of an MM argument. With a few more iterations and fixing a few more gaps (initially, Gemini claimed that the iterations of $\phi$ jumped directly to the maximum of the surrogate; upon my attempt to rewrite the logic, I discovered this was incorrect), we had the result you see above. But the interaction above contained a crucial easy-to-miss step, which can serve as our second principle:

\textbf{Principle 2: Beware of your context.} Transformer-based LLMs typically attend to their context (in this case, the chat history) on every turn. This means that ideas present in the context can in some sense ``pollute'' new outputs, even if you are hoping to steer the model away from these ideas. Remember, an LLM is in ``just'' a steerable distribution over text, with the context serving as the conditioning variables for the generation process. Thus, the model is generating for you text that is likely to follow the context, presumably modulated by some posttraining to steer it towards free-flowing interactions. (I do not work in posttraining, so I do not have intimate familiarity with such details). But for this reason, things in the context can be very hard to make the model ignore. If it has responded to you with an answer you don't like multiple times, you may be better off manually curating the context by starting a new chat, and feeding in your summary of the things you want the model to know, filtering out the things you wish for it to avoid.

One aspect of this interaction which I suppressed above was the generation of colab snippets to check that we could not generate counterexamples to the claimed monotonic ascent. The validity of these snippets (plus my recollection that this was empirically true during the writing of~\cite{denisov}) served as a critical grounding in continuing to push the model towards a proof. Here Gemini was reasonably inventive, numerically attempting to \emph{minimize the gain} in the vector $v$. Gemini generated the colab to follow, which I then altered slightly while running, to ensure the claimed ascent was not simply a low-dimensional phenomenon.

\begin{lstlisting}
import numpy as np
from scipy.linalg import eigh
from scipy.optimize import minimize
import matplotlib.pyplot as plt

def get_sym_sqrt(A):
    vals, vecs = eigh(A)
    # Numerical stability: Clip small eigenvalues
    vals = np.maximum(vals, 1e-14)
    return vecs @ np.diag(np.sqrt(vals)) @ vecs.T

def potential_J(v, M):
    v_sqrt = np.sqrt(v)
    D_sqrt = np.diag(v_sqrt)
    Inner = D_sqrt @ M @ D_sqrt
    
    # Tr(A^{1/2}) is sum(sqrt(eigenvalues(A)))
    vals = np.linalg.eigvalsh(Inner)
    vals = np.maximum(vals, 1e-14)
    
    return 2 * np.sum(np.sqrt(vals)) - np.sum(v)

def phi_map(v, M):
    v_sqrt = np.sqrt(v)
    D_sqrt = np.diag(v_sqrt)
    Inner = D_sqrt @ M @ D_sqrt
    S = get_sym_sqrt(Inner)
    return np.diag(S)

def run_adversarial_search(N=3):

    # Parameterize M = L L^T (to ensure PD) and v = exp(u) (to ensure v > 0)
    # Params: [u (size N), L_flat (size N*N)]
    
    def pack(v, L):
        return np.concatenate([np.log(v), L.flatten()])
    
    def unpack(params):
        u = params[:N]
        L_flat = params[N:]
        v = np.exp(u)
        L = L_flat.reshape(N, N)
        M = L @ L.T + 1e-5 * np.eye(N)
        return v, M
    
    def objective(params):
        v, M = unpack(params)
        v_next = phi_map(v, M)
        
        # We want to minimize gain
        gain = potential_J(v_next, M) - potential_J(v, M)
        return gain

    # Initial random guess
    A_init = np.random.randn(N, N)
    v_init = np.random.rand(N) + 0.1
    x0 = pack(v_init, A_init)
    
    # Run Optimization (L-BFGS-B handles smooth unconstrained problems well)
    res = minimize(objective, x0, method='L-BFGS-B',tol=1e-9)
    
    min_gain = res.fun
    
    if min_gain < -1e-8:
        raise ValueError(min_gain)
    else:
        pass

# Run the search

# Test on a few dimensions
run_adversarial_search(N=2)
run_adversarial_search(N=3)
run_adversarial_search(N=5)

\end{lstlisting}

This colab verification pattern I have found useful in several other mathematics-heavy contexts when working with AI models; e.g., if the model claims to prove convexity of a certain function, the quickest route to falsifying this claim is often to simply ask the model to generate a code snippet which tests this claim numerically. This leads to our third and final principle:

\textbf{Principle 3: Ask for empirical verification of theory.} If this test fails, it is often enough to feed this evidence back to the model to encourage it to revise its approach. Given the tendency of models to either hallucinate or suppress reasoning steps, quick empirical verification of claimed theoretical statements has saved me much time in the past.

I have hope that other mathematicians and scientists can save themselves time and frustration, and achieve a fair sense of what the current models are capable of, by following and expanding on these three principles.

\section{On the subtleties of checking correctness}

In an earlier version of this note, before posting publicly, there was an incorrect ``proof'' of~\cref{lem:concave} which slipped past me.

It is reproduced below.

\begin{fakeproof}
Let $\Psi(v) = D_v^{1/2} M D_v^{1/2}$ and note that $\Psi(v) = (\sqrt{v}\sqrt{v}^\top) \odot M$. Since the square root function is strictly concave on $\mathbb{R}_{++}$, the map $v \mapsto \sqrt{v}\sqrt{v}^\top$ is strictly operator concave. Because $M \succ 0$, the Schur Product Theorem implies that $\Psi(v)$ is also strictly operator concave.

Finally, since the trace-square-root functional $X \mapsto \Tr(X^{1/2})$ is operator monotone and strictly concave on the positive definite cone, the composition $J(v)$ is strictly concave. The linear term $-\sum v_i$ has a zero Hessian, thus the strict concavity of the first term yields $\nabla^2 J(v) \prec 0$ for all $v \in \R_{++}^N$.
\end{fakeproof}

The fundamental flaw with this argument is in the claim that $v \mapsto \sqrt{v}\sqrt{v}^\top$ is operator concave. A simple example choosing $v = (1, 4)$ and $w = (4, 1)$ is enough to contradict this claim. However, I did not spot it. Almost certainly I had asked for some kind of numerical verification of the claim of~\cref{lem:concave}, similar to the snippet above, and this led me to be a bit too trusting in the details of the argument.

Later on, I caught a subtle mismatch between the definition of $A$ used in~\cref{sec:minorizer_construction} and that which corresponds to the actual literal potential function. (The model used $D_v^{1/2}M^{1/2}$ instead of $M^{1/2}D_v^{1/2}$). Technically this was only a missed logical step, and not incorrect, because of the equivalence of the eigenvalues of $AB$ and $BA$ as noted in the (new) proof of~\cref{lem:concave}, but it resulted in some strange gymnastics (e.g. the use of the left polar factorization instead of the more traditional right polar factorization). I rewrote the argument in the current form you find it, which I find much more straightforward.

In some ways, each of these is an effective illustration of our principle 1 above. But it also highlights a real difficulty in using AI to accelerate mathematics in this way, at least, without coupling with formal methods. On the other hand, it is a belief of mine (shared by at least one other mathematician I know) that every math paper has gaps; the only question that counts is whether or not those gaps matter. Neither of these gaps matter. And indeed, once Gemini presented me with the variational characterization of the nuclear norm~\cref{lem:var_nuc}, I knew the problem was solved, and the rest was details.

Time will tell, I think, whether the \emph{ease} of producing a gap with AI-assisted mathematics will proliferate their relative occurrence, and whether the human checking will be overwhelmed by the volume of argumentation and its apparent simplicity, or whether this is a phantom concern that will either be solved by better models or improved ways of working.
\end{document}